\documentclass{article}

\usepackage{amssymb,amsmath, amsthm}
\usepackage{url}

\usepackage{setspace}
\usepackage{graphicx} 
\usepackage{color}  
\usepackage{pifont}
\usepackage[mathscr]{euscript}

\usepackage[T1]{fontenc}
\usepackage{epigraph}

\usepackage[lighttt]{lmodern}
\usepackage{caption}
\usepackage{booktabs} 
\usepackage{geometry}

\usepackage{sectsty}   

\usepackage{titlesec}

\newtheorem{definition}{Definition}[section]
\newtheorem{proposition}{Proposition}[section]
\newtheorem{theorem}{Theorem}[section]

\newcommand{\lo}{\textit{\textbf{l}}_R}
\newcommand{\up}{\textit{\textbf{u}}_R}
\newcommand{\low}{\textit{\textbf{l}}}
\newcommand{\upp}{\textit{\textbf{u}}}
\newcommand{\ut}{\mathtt{u}_R}

\begin{document}

	\begin{center}
		\doublespacing{\textbf{\huge{On a Well-behaved Relational Generalisation of Rough Set Approximations}}}	\\
		
		\vspace{0.8cm}
		Alexa Gopaulsingh\\
		
		\vspace{1mm}
	\singlespacing{	Central European University\\
		Budapest, Hungary\\
		\ttfamily{\bfseries{gopaulsingh\_alexa@phd.ceu.edu}}}\\
		\end{center} 
		
\vspace{1cm}

\noindent \textbf{Abstract.} We examine  non-dual relational extensions of rough set approximations and find an extension which satisfies surprisingly many of the usual rough set properties. We then use this definition to give an explanation for an observation made by Samanta and Chakraborty in their recent paper [P. Samanta and M.K. Chakraborty. Interface of rough set systems and modal logics: A survey. \textit{Transactions on Rough Sets XIX}, pages 114-137, 2015].

\vspace{4mm} \noindent \textbf{Keywords:} Relational generalisation $\cdot$ Relational extension $\cdot$ Non-dual $\cdot$ Rough Set Theory
	
	\section{Relational Generalisations}
	Given the definition of rough sets, it is natural to ask  what happens if we relax the equivalence relation to an arbitrary relation or  other special relation. This has been well studied in the literature, see \cite{M1, R2, R3, R4, R5, R6, R7, R8}. In Section 3 we will recall a standard relational generalisation and compare properties which special relations satisfy under this generalisation. We  will use a table similar to one in \cite{PandC} and from it, one can obtain three characterising properties of  reflexive relations, symmetric relations and transitive relations each. In \cite{R8}, Yao did work which, though not directly stated, essentially implies these three propositions using operator theory. So, here we mention them so they can be seen directly in this simple form and also for comparison purposes with  Section 4. Then we give brief, direct proofs which does not use operator terminology. 
	
	Next, we will observe that when generalising from rough sets based on equivalence relations to general relations, in the literature one often (though not only) sees either predecessor sets for both approximations or more commonly successor sets for both approximations. The approximation properties with predecessor and successor sets has been combinatorically mentioned in \cite{CH1} and well as \cite{R7}. In \cite{R7}, Yao also groups lower and upper approximations together which uses both predecessor sets each or successor sets each or which uses both each. He made these pairs most probably because they form dual operators as such. Non-dual systems are considered in the literature in for example, \cite{PandC}. We continue to go in this direction. In Section 2.2, we show that if we are willing to give up duality, and combine a lower approximation which uses successor sets and an upper approximation which uses predecessor sets, we can find a generalisation which pretty much satisfies everything else except for duality,  for pre-order (reflexive and transitive) relations. Here too, we construct a table and find a characterising property of transitive relations. From comparison between the two tables, we will see that none of the rough set generalisations along the special relations examined in the first table, satisfies so many properties as the pre-order generalisation of the non-dual extension in the second table. We believe that this makes this form of generalisation, and more specially the pre-order case of it, worthy of further consideration. In this direction, we will give two uses of this generalisation. One is that a covering operator mentioned in \cite{PandC} is a special case of this operator, which explains why Samanta and Chakraborty noticed that it satisfied so many rough set properties and they remarked that it deserved to be investigated further. Here, we see the reason for that well-behaved operator by placing it in this more general context. Lastly we, will give an example of a logical interpretation of this operator which indicates that results about it can be nicely transferred to other areas.

		Rough set theory has tremendous practical applications. This comes from the calculation of reducts and decision rules for data. The data is mined to extract decision rules of manageable size (i.e. attribute reduction) so predictions can be made. It has been argued that rough set theory can be used to make decisions on the  data in the absence of major prior assumptions in \cite{Baye}.  From this perspective, it is perhaps not so surprising that this leads to an explosion of applications. Hence rough set analysis adds to the tools of Bayes' Theorem and regression analysis for feature selection and pattern recognition in data mining  \cite{DM4, DM1, DM2, DM3, DM5, DM6, DM7, DM8}.  Applications include in medical databases \cite{MD1, MD2, MD3, MD4, MD5, MD6, MD7}, artificial intelligence and machine learning \cite{AC1, AC2, AC3, AC4, AC5, AC6, AC7},  engineering  \cite{EN1, EN2, EN3, EN4, EN5} and cognitive science  \cite{CG1, CG2, CG3, CG4, CG5}. In fact, it has been argued by Yao in \cite{TSid} that there is an imbalance in the literature of rough sets between the conceptual development of the theory and the practical computational development. Here, he claims that the computational literature now far outweigh the conceptual theoretical literature and that it would be useful for the field if this imbalance were somewhat corrected. He started his suggestion in \cite{TSid} where he gave a conceptual example of reducts which unify three different looking reduct definitions used in the literature. We agree that more work of this type would be  helpful in organising and making a more coherent map of the huge mass of rough set literature which is present. In this direction, in subsection 4.1.2 we briefly state an interpreted logical connection for the well-behaved relational extension of rough set approximations examined in section 4. 
	
	\section{Pawlak's Rough Sets}  
	We go over the definition of rough set approximations. These definitions and basic notions  can be found in   \cite{PZ}.  Let $V$ be a set and $E$ be an equivalence relation on $V$. Let $V/E$ denote the set of equivalence classes of $E$. A set $X \subseteq V$ is said to be \textit{E-exact} if it is equal to a union of some of the equivalence classes of $E$. If $X$ cannot be represented in this way, it is said to be \textit{E-inexact} or \textit{E-rough} or simply \textit{rough} if the equivalence relation under consideration is clear. In this case, we may approximate it with two exact sets, the lower and upper  approximations respectively as defined below:
	\begin{center} 
		$ \low_E (X)	 = \{ x \in V\ | \ [x]_E \subseteq X \} $, 
	\end{center}
	\begin{equation} \label{eq:1}
	\upp_E (X) = \{  x\in V \ | \ [x]_E  \cap X \neq \emptyset \}. 
	\end{equation}
	Equivalently, instead of a pointwise definition we may use a granule based definition:
	\begin{center}
		$ \low_E (X)= \bigcup \{ Y\subseteq V/E \ | \ Y \subseteq X  \}, $
	\end{center}
	\begin{equation} \label{eq:2}
	\upp_E (X)= \bigcup \{  Y \subseteq V/E \ | \ Y \cap X \neq \emptyset  \}. 
	\end{equation}

	The pair $(V, E)$ is known as an \textit{approximation space}. 
	
	Many times, several equivalence relations are considered over on set. A \textit{knowledge base}, $K = (V, \mathscr{E} )$ is defined with $\mathscr{E} $ being a family of equivalence relation over $V$. If $ \mathscr{P} \subseteq \mathscr{E} $, the $\bigcap \mathscr{P}$ is an equivalence relation as well. The intersection of all equivalence relations belonging to $\mathscr{P}$ is denoted by $IND(\mathscr{P}) = \bigcap \mathscr{P}$. This is known as the \textit{indiscernibility relation} over $\mathscr{P}$.\\  For further basic concepts of rough set theory, see \cite{PZ}.

	\subsection{List of properties satisfied by Rough Sets based \\ on Equivalence Relations} 
	
	In Pawlak's book, see \cite{PZ}, he lists these properties of rough sets based on equivalence relations which we repeat here. Let $V$ be the domain of discourse and $X, Y \subseteq V$. Then, the following holds:\\

	\begin{onehalfspace}
		\noindent $1) \low_E (X) \subseteq X \subseteq \upp_E (X),$
		
		\vspace{2mm}
		\noindent $2) \low_E (\emptyset) = \upp_E (\emptyset) = \emptyset; \quad \low_E (V) = \upp_E (V)= V,$
		
		\vspace{2mm}
		\noindent $ 3) \upp_E (X \cup Y) = \upp_E (X) \cup \upp_E (Y),$
		
		\vspace{2mm}
		\noindent $ 4) \low_E (X \cap Y) = \low_E (X) \cap \low_E (Y),$
		
		\vspace{2mm}
		\noindent $ 5)  X \subseteq Y \Rightarrow \low_E (X) \subseteq \low_E (Y),$
		
		\vspace{2mm}
		\noindent $ 6)  X \subseteq Y \Rightarrow \upp_E (X) \subseteq \upp_E (Y),$
		
		\vspace{2mm}
		\noindent $ 7) \low_E (X\cup Y) \supseteq \low_E (X) \cup \low_E (Y),$
		
		\vspace{2mm}
		\noindent $ 8) \upp_E (X\cap Y) \supseteq \upp_E (X) \cap \upp_E (Y),$
		
		\vspace{2mm}
		\noindent $ 9) \low_E (-X) = -\upp_E (X),$
		
		\vspace{2mm}
		\noindent $ 10) \upp_E (-X) = -\low_E (X),$
		
		\vspace{2mm}
		\noindent $ 11) \low_E (\low_E (X)) = \upp_E (\low_E (X)) = \low_E (X),$
		
		\vspace{2mm}
		\noindent $ 12) \upp_E (\upp_E (X)) = \low_E (\upp_E (X)) = \upp_E (X).$
	\end{onehalfspace}

	\section{Standard Dual Relational Generalisation of \\ Rough Set Approximations}
	 
	\vspace{2mm}
	We will use a set-up similar to that of Yao, in \cite{GRSM}. Here, Yao defined  a relational generalisation of rough set approximation as follows: Let $R$  be a binary relation on a set $V$ i.e. $ R \subseteq V\times V.$ First, we state the notion of a \emph{successor neighbourhood} of an element $x \in V,$ $R_s(x)$:     
	\begin{equation} 
	R_s(x) = \{ y \in V \ | \ xRy \} 
	\end{equation}
	This was then used to define the corresponding notion of lower and upper approximation operators as below.  We will use similar notation as in Section 1.2 but note that the subscript here can be any relation not just an equivalence relation. For $X \subseteq V$ we have:

	\begin{center}
		$\lo (X) = \{ x\ | \ R_s(x) \subseteq X \} $
	\end{center}
	\begin{equation}
	\up (X)= \{ x \ | \ R_s(x) \cap X \neq \emptyset  \}
	\end{equation}

	In the Table 1, we have enlisted the properties that different special relations may satisfy. A similar such table was given in \cite{PandC} and we use it here  for comparison with  the non-dual relational generalisation examined in the following section. This is a  table of properties versus different types of relations. A box is marked with a tick if all relations of the type corresponding to its column satisfies the property stated in its row and is marked with a cross otherwise. Different properties follow for different special relations. Let $r$, $s$, $t$, be subscripts which denote when a relation is reflexive, symmetric and transitive respectively and their combinations denote the conjunction of these properties. Let the subscript $ser$ denote a serial relation i.e. a relation in which every element has a successor.

	\begin{table}
		
		\begin{tabular}{| l | l | l | l | l | l | l | l | l |l |}
			\hline
			\hfill & $R$ &$R_r$ & $R_s$ & $R_t$ & $R_{rs}$ & $R_{rt}$ & $R_{st}$ & $R_{rst}$ &$R_{ser}$ \\ \hline
			1. Duality of $\lo (X)$, $\up (X)$ & \ding{51} & \ding{51} & \ding{51} & \ding{51}  & \ding{51} & \ding{51}  & \ding{51} & \ding{51} & \ding{51} \\ \hline
			
			2. $\lo (\emptyset)$ = $\emptyset$&\ding{53} & \ding{51} & \ding{53} & \ding{53}  & \ding{51} & \ding{51}  & \ding{53} & \ding{51} & \ding{51} \\ \hline
			
			3. $\emptyset$ = $\up (\emptyset) $ &\ding{51} & \ding{51} & \ding{51} & \ding{51}  & \ding{51} & \ding{51}  & \ding{51} & \ding{51} & \ding{51} \\ \hline
			
			4. $\lo (V) = V$ &\ding{51} & \ding{51} & \ding{51} & \ding{51}  & \ding{51} & \ding{51}  & \ding{51} & \ding{51} & \ding{51} \\ \hline
			
			5. $\up (V) = V$&\ding{53} & \ding{51} & \ding{53} & \ding{53}  & \ding{51} & \ding{51}  & \ding{53} & \ding{51} & \ding{51}  \\ \hline

			6. $\lo (X) \subseteq X$  &\ding{53} & \ding{51} & \ding{53} & \ding{53}  & \ding{51} & \ding{51}  & \ding{53} & \ding{51} & \ding{53}  \\ \hline
			
			7. $X\subseteq \up (X) $&\ding{53} & \ding{51} & \ding{53} & \ding{53}  & \ding{51} & \ding{51}  & \ding{53} & \ding{51} & \ding{53}  \\ \hline

			8. $ X \subseteq Y \Rightarrow \lo (X) \subseteq \lo (Y)$&\ding{51} & \ding{51} & \ding{51} & \ding{51}  & \ding{51} & \ding{51}  & \ding{51} & \ding{51} & \ding{51} 
			\\ \hline
			
			9. $ X \subseteq Y \Rightarrow \up (X)\subseteq \up (Y)$&\ding{51} & \ding{51} & \ding{51} & \ding{51}  & \ding{51} & \ding{51}  & \ding{51} & \ding{51} & \ding{51} 
			\\ \hline
			
			10. $\up (X\cup Y) = \up (X) \cup \up (Y)$ &\ding{51} & \ding{51} & \ding{51} & \ding{51}  & \ding{51} & \ding{51}  & \ding{51} & \ding{51} & \ding{51} \\ \hline 
			
			11. $\lo (X\cap Y) = \lo (X) \cap \lo (Y) $&\ding{51} & \ding{51} & \ding{51} & \ding{51}  & \ding{51} & \ding{51}  & \ding{51} & \ding{51} & \ding{51} \\ \hline
			
			12. $\lo (X\cup Y) \supseteq \lo (X)\cup \lo (Y) $ &\ding{51} & \ding{51} & \ding{51} & \ding{51}  & \ding{51} & \ding{51}  & \ding{51} & \ding{51} & \ding{51} 
			\\ \hline
			
			13. $\up (X\cap Y) \supseteq \up (X) \cap \up (Y) $ &\ding{51} & \ding{51} & \ding{51} & \ding{51}  & \ding{51} & \ding{51}  & \ding{51} & \ding{51} & \ding{51} 
			\\ \hline

			14. $ \lo (\lo (X)) \subseteq \lo (X) $ &\ding{53} & \ding{51} & \ding{53} & \ding{53}  & \ding{51} & \ding{51}  & \ding{53} & \ding{51} & \ding{53}  \\ \hline
			
			15. $ \lo (\lo (X)) \supseteq  \lo (X) $ &\ding{53} & \ding{53} & \ding{53} & \ding{53}  & \ding{53} & \ding{51}  & \ding{53} & \ding{51} & \ding{53} 
			\\ \hline

			16. $ \up (\lo (X)) \subseteq \lo (X) $ &\ding{53} & \ding{53} & \ding{53} & \ding{53}  & \ding{53} & \ding{53}  & \ding{53} & \ding{51} & \ding{53} 
			\\ \hline
			
			17. $ \up (\lo (X)) \supseteq \lo (X) $ &\ding{53} & \ding{51} & \ding{53} & \ding{53}  & \ding{51} & \ding{51}  & \ding{53} & \ding{51} & \ding{53} \\ \hline

			18. $ \up (\up (X)) \subseteq  \up (X) $ &\ding{53} & \ding{53} & \ding{53} & \ding{51}  & \ding{53} & \ding{51}  & \ding{51} & \ding{51} & \ding{53} \\ \hline
			
			19. $ \up (\up (X)) \supseteq \up (X) $ &\ding{53} & \ding{51} & \ding{53} & \ding{53}  & \ding{51} & \ding{51}  & \ding{53} & \ding{51} & \ding{53} \\ \hline

			20. $ \lo (\up (X)) \subseteq  \up (X) $ &\ding{53} & \ding{51} & \ding{53} & \ding{53}  & \ding{51} & \ding{51}  & \ding{53} & \ding{51} & \ding{53}  \\ \hline

			21. $\up (X) \subseteq \lo (\up (X))$   &\ding{53} & \ding{53} & \ding{53} & \ding{53}  & \ding{53} & \ding{53}  & \ding{53} & \ding{51} & \ding{53}\\ \hline
			
			22. $ X \subseteq \lo (\up (X))$  &\ding{53} & \ding{53} & \ding{51} & \ding{53}  & \ding{51} & \ding{53}  & \ding{51} & \ding{51} & \ding{53} \\ \hline

			23. $\up (\lo (X))\subseteq X$  &\ding{53} & \ding{53} & \ding{51} & \ding{53}  & \ding{51} & \ding{53}  & \ding{51} & \ding{51} & \ding{53}\\ \hline

		\end{tabular} 
		
		\caption{Properties satisfied by the general approximation operators for different special relations }
		\label{table:table1} 
	\end{table}
	
	We note that if a property is satisfied by any general relation, i.e. there is a tick in the first column, then the full row corresponding to that property is ticked. Also, if for example, if some property is satisfied by a reflexive relation i.e.  $R_r$ ticked then we can immediately deduce that $R_{rs}, \ R_{rt},$ and $R_{rst}$ should be ticked. Similarly, for other special relations. So, often only the first few boxes of a row needs to be figured out before the whole row can be deduced. For example, consider the case of $\up (\up (X)) \subseteq \up (X)$ in the 18th row  and the  $R_t$ column, i.e for a transitive relation. We now briefly prove this.  Suppose that $x \in \up (\up (X)),$ i.e. $R_s(x) \cap \up (X) \neq \emptyset.$ Let $v$ be in this intersection. Then $v \in R_s(x)$ and  $R_s(v) \cap X \neq \emptyset.$ So let $t \in R_s(v) \cap X.$ Since $R$ is transitive then we also have that $ t\in R_s(x).$ Hence, $t\in R_s(x) \cap X$ and  $R_s(x) \cap X \neq \emptyset.$ Therefore, $ x \in \up (X)$ and $\up(\up(X)) \subseteq \up(X).$ It follows that boxes corresponding to  $R_t, \ R_{rt}, \ R_{st}$ and $R_{rst}$ are ticked.  Counter-example cases can be made for the boxes marked with a cross. We note that not all the rows are independent. For example, $\lo (\lo(X)) \subseteq \lo (X)$ in the 14th row is a special case of $\lo (X) \subseteq  X$ in the 6th row. However, we wanted to include both sides of the idempotent  equation, consisting of rows 14 and 15. Similar considerations go for the rest of the table.
	
	Examining Table 1, we observe a few things. Duality, property 1. as well as properties, 2, 3, 8-13 hold for arbitrary relations. The table also hints at the upcoming 3 propositions. In rows 6, 23 and 18, we see properties which hold for reflexive (but not symmetric and transitive), symmetric (but not reflexive and transitive) and transitive (but not reflexive and symmetric relations) respectively. Forming the table helps to see what possibilities would be promising to try to see if it holds both ways and it can be seen that not only do properties 6, 18. and 23. imply that a relation is reflexive, symmetric and transitive respectively but that the converses hold as well.  In the \cite{R8} paper, these propositions can be deduced from examinations of algebraic operators. Here, we give brief direct proofs of them and in the next  section we will compare these results with what can be obtained for the case of the non-dual generalisation examined.

	\vspace{2mm}
	\begin{proposition} \label{p1}
		Let $V$ be a set and $R$ a relation on $V.$ Then  $ \lo(X) \subseteq X $ for all $ X \subseteq V $ iff $R$ is reflexive.
	\end{proposition}
	
	\begin{proof}
		$\Leftarrow$ is straightforward so we prove the converse.	We prove it by the contrapositive. Suppose that $R$ is not reflexive. Then there exists a witness $x \in V$ such that $(x,x) \not\in R.$ Consider the set $Y= R_s(x).$ Now by definition $x \in \lo (Y)$ but by assumption $x \not\in Y.$ Hence  $ \lo (Y) \not\subseteq Y.$ 
	\end{proof} 
	
	\noindent \textbf{Remark 2.3.1} In \cite{R5}, Zhu also noted the above proposition for characterising approximations for reflexive relations as well as another proposition which characterises reflexive approximations using property 5. in  Table 1 instead of property 4.

	\begin{proposition} \label{p2}
		Let $V$ be a set and $R$ a relation on $V.$	Then $ \up  (\lo (X)) \subseteq X$ for all $X \subseteq V$ iff $R$ is symmetric. 
	\end{proposition}
	
	\begin{proof}
		$\Leftarrow$ is straightforward so we prove the converse.	We prove it by the contrapositive. Suppose that $R$ is not symmetric. Then there exits witnesses $x, y \in V$ such that $(x,y) \in R$ but $(y,x) \not\in R.$	Consider the set $ Y = R_s(y).$ By definition we have that $ y \in \lo (Y)$. Since $(x,y) \in R$ then $x \in \up  (\lo (Y))$ and since $(y,x) \not\in R$ then $x$ is not in $Y.$ Therefore, $ \up (\lo (Y)) \not\subseteq Y.$ Hence the result. 
	\end{proof}

	\begin{proposition} \label{p3}
		Let $V$ be a set and $R$ a relation on $V.$ Then $ \up (\up (X)) \subseteq \up (X) $ for all $ X \subseteq V $ iff $R$ is transitive.	
	\end{proposition}
	
	\begin{proof}
		$\Leftarrow$ is straightforward so we prove the converse.	We prove it by the contrapositive. Suppose that $R$ is not transitive. Then there exists witnesses $x, y$ and $z \in V$ such that $(x,y), (y,z) \in R$ but $(x,z) \not\in R.$ Consider the set $Z = \{z\}.$ Then $\up (Z)$ contains $y$ and hence $ \up (\up (Z)) $ contains $x$ but since $(x,z) \not\in R$, $x$ is not in $\up (Z).$ Hence $ \up (\up (Z) \not\subseteq \up (Z). $ The result follows.
	\end{proof}
	
	\newpage 
	
	\begin{theorem}
		Let R be a relation on a set V. For all $X \subseteq V$, then 
		
		\vspace{2mm}
		\noindent (i) $ \lo(X) \subseteq X $,\\
		(ii) $ \up (\lo (X)) \subseteq X$ and \\
		(iii)  $ \up (\up (X)) \subseteq \up (X). $ 
		
		\noindent all hold iff $R$ is an equivalence relation. 
		
	\end{theorem}
	
	\begin{proof} 
		This is an immediate corollary of Proposition \ref{p1}, Proposition \ref{p2} and Proposition \ref{p3}.
	\end{proof}  
	
	\noindent \textbf{Remark 2.3.2} We note that if we replace property (i) in the above theorem with the property 5. from Table 1, namely $ X \subseteq \up (X),$ then we get a similar alternative theorem.\\

	\section{Non-Dual Relational Generalisation of Rough\\  Set Approximations} 
	
	\vspace{2mm}
	Here we examine the properties of a non-dual coupling of lower and upper relational approximations.  Analogous to the definition given in equation (2.1), we now give the definition of a \emph{predecessor neighbourhood} of an element $x \in V,$ $R_p(x),$ as follows:
	\begin{equation}
	R_p(x) = \{ y \in U  | \ yRx \}.
	\end{equation}
	
	\noindent We  will use the lower and upper approximation definitions as follows:
	\begin{center}
		$\lo (X)  = \{ x\ | \ R_s(x) \subseteq X \} $
	\end{center}
	\begin{equation}
	\ut (X)  = \{ x \ | \ R_p(x) \cap X \neq \emptyset  \}
	\end{equation}
	
	\noindent To emphasize that this is a different upper approximation than the standard generalisation, we use a different font to denote the upper approximation, $\ut.$ In this case, the upper approximation of a set consists of all the successors of elements in that set instead of all the predecessors of elements in that set as in the standard generalisation. 
	
	Different properties follow for different special relations.  Again, let $r$, $s$, $t$ be subscripts which denote when a relation is reflexive, symmetric and transitive and respectively and their combinations denote the conjunction of these properties. Also, let the subscript $ser$ a serial relation. 
	 
	Consider Table 2.  Like before, properties 3,4, 8-13 hold for arbitrary relations. However, here we see that duality, property 1. of the table does not hold for arbitrary relations like it does for the standard relational generalisation. On the hand, $X \subseteq \lo (\ut (X))$ and $\ut (\lo (X) ) \subseteq X,$ properties 22. and 23. respectively, does hold for arbitrary relations unlike for the case of the standard relational generalisation. We can also see that the  $R_{rst}$ column, i.e. the column corresponding to  an equivalence relation satisfies all of the properties as expected. However,  here there is another column of interest which we would like to draw your attention to, namely the column corresponding to $R_{rt}$.  This corresponds to a pre-order and we observe that this satisfies all of the examined rough set properties except the duality of the lower and upper approximation operators. This feature makes it quite interesting and worthy of  further consideration.     
			
			In Section 2.2, we mentioned characterising properties  of the standard relational generalisation which imply $R$ is an equivalence relation. Here, we have characterising properties of the non-dual relational generalisation which imply that $R$ is a pre-order. 
			
			\newpage

	\begin{table}
		
		\begin{tabular}{| l | l | l | l | l | l | l | l | l |l |l }
			\hline
			\hfill & $R$ &$R_r$ & $R_s$ & $R_t$ & $R_{rs}$ & $R_{rt}$ & $R_{st}$ & $R_{rst}$ &$R_{ser}$ \\ \hline
			1. Duality of $ \lo(X) $, $\ut (X)$ & \ding{53} & \ding{53} & \ding{51} & \ding{53}  & \ding{51} & \ding{53}  & \ding{51} & \ding{51} & \ding{53} \\ \hline
			
			2. $\lo (\emptyset)$ = $\emptyset$&\ding{53} & \ding{51} & \ding{53} & \ding{53}  & \ding{51} & \ding{51}  & \ding{53} & \ding{51} & \ding{51} \\ \hline
			
			3. $\emptyset$ = $\ut (\emptyset) $ &\ding{51} & \ding{51} & \ding{51} & \ding{51}  & \ding{51} & \ding{51}  & \ding{51} & \ding{51} & \ding{51} \\ \hline
			
			4. $\lo (V) = V$ &\ding{51} & \ding{51} & \ding{51} & \ding{51}  & \ding{51} & \ding{51}  & \ding{51} & \ding{51} & \ding{51}  \\ \hline
			
			5. $\ut (V) = V$&\ding{53} & \ding{51} & \ding{53} & \ding{53}  & \ding{51} & \ding{51}  & \ding{53} & \ding{51} & \ding{53}  \\ \hline
			
			6. $\lo (X)\subseteq X$  &\ding{53} & \ding{51} & \ding{53} & \ding{53}  & \ding{51} & \ding{51}  & \ding{53} & \ding{51} & \ding{53}  \\ \hline
			
			7. $X\subseteq \ut (X)$&\ding{53} & \ding{51} & \ding{53} & \ding{53}  & \ding{51} & \ding{51}  & \ding{53} & \ding{51} & \ding{53}  \\ \hline
			
			8. $ X \subseteq Y \Rightarrow \lo (X) \subseteq \lo (Y)$&\ding{51} & \ding{51} & \ding{51} & \ding{51}  & \ding{51} & \ding{51}  & \ding{51} & \ding{51} & \ding{51} \\ \hline
			
			9. $ X \subseteq Y \Rightarrow \ut (X) \subseteq \ut (Y)$&\ding{51} & \ding{51} & \ding{51} & \ding{51}  & \ding{51} & \ding{51}  & \ding{51} & \ding{51} & \ding{51} \\ \hline

			10. $\ut (X\cup Y) = \ut (X) \cup \ut (Y)$ &\ding{51} & \ding{51} & \ding{51} & \ding{51}  & \ding{51} & \ding{51}  & \ding{51} & \ding{51} & \ding{51} \\ \hline
			
			11. $\lo (X\cap Y) = \lo (X) \cap \lo (Y) $&\ding{51} & \ding{51} & \ding{51} & \ding{51}  & \ding{51} & \ding{51}  & \ding{51} & \ding{51} & \ding{51} \\ \hline

			12. $\lo (X\cup Y)  \supseteq \lo (X) \cup \lo (Y) $ &\ding{51} & \ding{51} & \ding{51} & \ding{51}  & \ding{51} & \ding{51}  & \ding{51} & \ding{51} & \ding{51} \\ \hline
			
			13. $\ut (X\cap Y) \supseteq \ut (X) \cap \ut (Y) $ &\ding{51} & \ding{51} & \ding{51} & \ding{51}  & \ding{51} & \ding{51}  & \ding{51} & \ding{51} & \ding{51} \\ \hline

			14. $ \lo (\lo (X)) \subseteq \lo (X) $ &\ding{53} & \ding{51} & \ding{53} & \ding{53}  & \ding{51} & \ding{51}  & \ding{53} & \ding{51} & \ding{53}  \\ \hline
			
			15. $ \lo (\lo (X)) \supseteq  \lo (X) $ &\ding{53} & \ding{53} & \ding{53} & \ding{53}  & \ding{53} & \ding{51}  & \ding{53} & \ding{51} & \ding{53}  \\ \hline

			16. $ \ut (\lo (X)) \subseteq  \lo (X) $ &\ding{53} & \ding{53} & \ding{53} & \ding{51}  & \ding{53} & \ding{51}  & \ding{53} & \ding{51} & \ding{53} \\ \hline
			
			17. $ \ut (\lo (X)) \supseteq \lo (X) $ &\ding{53} & \ding{51} & \ding{53} & \ding{53}  & \ding{51} & \ding{51}  & \ding{53} & \ding{51} & \ding{53}  \\ \hline

			18. $ \ut (\ut (X)) \subseteq  \ut (X) $ &\ding{53} & \ding{53} & \ding{53} & \ding{51}  & \ding{53} & \ding{51}  & \ding{51} & \ding{51} & \ding{53} \\ \hline
			
			19. $ \ut (\ut (X)) \supseteq \ut (X) $ &\ding{53} & \ding{51} & \ding{53} & \ding{53}  & \ding{51} & \ding{51}  & \ding{53} & \ding{51} & \ding{53} \\ \hline

			20. $ \lo (\ut (X)) \subseteq  \ut (X)$ &\ding{53} & \ding{51} & \ding{53} & \ding{53}  & \ding{51} & \ding{51}  & \ding{53} & \ding{51} & \ding{53}  \\ \hline
			
			21. $ \lo (\ut (X)) \supseteq \ut (X) $ &\ding{53} & \ding{53} & \ding{53} & \ding{51}  & \ding{53} & \ding{51}  & \ding{51} & \ding{51} & \ding{53}
			\\ \hline
			
			22. $ X \subseteq \lo (\ut (X))$ &\ding{51} & \ding{51} & \ding{51} & \ding{51}  & \ding{51} & \ding{51}  & \ding{51} & \ding{51} & \ding{51} \\ \hline

			23 . $\ut (\lo (X)) \subseteq X$ &\ding{51} & \ding{51} & \ding{51} & \ding{51}  & \ding{51} & \ding{51}  & \ding{51} & \ding{51} & \ding{51}  \\ \hline
			
		\end{tabular} 
		\caption{Properties satisfied by the alternative general approximation operators for different special relations }
		\label{table: table2}
	\end{table}

	\begin{proposition} \label{p4}
		Let $V$ be a set and $R$ a relation on $V.$ Then $\ut (X) \subseteq \lo (\ut (X)) $ for all $X \subseteq V$ iff  $R$ is transitive.	
	\end{proposition}	
	
	\begin{proof}
		$\Leftarrow$ is straightforward so we prove the converse.	We prove it by the contrapositive. Suppose that $R$ is not transitive. Then there exists witnesses $x, y$ and $z \in V$ such that $(x,y), (y,z) \in R$ but $(x,z) \not\in R.$ Consider the set $Y = \{ x\}.$ Then since $(x,y) \in R, $ we have that $y \in \ut (Y)$ and since $z$ is a successor of $y$ but $(x,z) \not\in R$ then $y \not\in \lo (\ut (Y)).$ Hence, $\ut (Y) \not\subseteq \lo (\ut (X))$ and the result follows.
	\end{proof}

	\begin{theorem} 
		Let R be a relation on a set V. For all $X \subseteq V$, then
		
		\vspace{2mm}
		\noindent 	(i) $ \lo (X) \subseteq X $ and\\
		(ii)  $\ut (X) \subseteq \lo (\ut (X)) $ 
		
		\noindent  both hold iff $R$ is a pre-order. 
		
	\end{theorem}
	
	\begin{proof}
		This is an immediate corollary of Proposition \ref{p1} and Proposition \ref{p4}.
	\end{proof}

	\subsection{Applications of the Non-Dual Relational Generalisation}
	\subsubsection{More general context for a special operator satisfying almost all rough set properties}

	An investigation suggested in \cite{PandC} asked the question why a certain $C_t$ operator defined in that paper, satisfies so many properties of rough approximation operators based on equivalence relations. Here, amongst other things, they considered covering generalisations of rough sets and they defined the neighbourhood of an element $x$ as all the the intersection of cover sets which contain $x.$ That is:
	\begin{definition} Let $\mathscr{C} = \{C_i: \ i \in I\}$ be a covering of $V$. Then a neighbourhood of a point $x \in V$ is given by:
		\begin{equation*}
		N(x) = \bigcap\limits_{i \in I} \{ C_i \in \mathscr{C} \ | \ x \in C_i\}.
		\end{equation*}
	\end{definition}
	
	\noindent 	Now we recall from that paper,  special lower and upper approximation operators, in their notation, $\underline{C_t}$ and $\overline{C_t},$ which satisfies all of their mentioned properties of approximation operators based on equivalence relations, except duality. They mentioned that this made this lower and upper approximation pair of operators worthy of further investigation. Here, we show that their operator is a special case of the non-dual relational generalisation which we examined in the previous section. First we recall their defined approximation operators below:
	\begin{definition} Let $V$ be the domain and for $x \in V,$  a set $D \subseteq V$ is said to be \emph{definable} if $D = \bigcup\limits_{x \in D} N(x).$ The collection of definable sets is denoted by,  $\mathfrak{D} = \{ D \subseteq V\ | \ D \text{ is definable}  \}.$ The lower and upper approximation operators, $\underline{C_t},$ $\overline{C_t}$, respectively, given in \cite{PandC} are as follows:
		\begin{equation*}
		\begin{split}
		\underline{C_t}(X) & = \bigcup \{ D \in \mathfrak{D} \ | \ D \subseteq X \}\\
		& = \bigcup \{ N(x) \ | \ N(x)  \subseteq X \},
		\end{split}
		\end{equation*}
		\begin{equation}
		\begin{split}
		\overline{C_t}(X)  & = \bigcap \{ D \subseteq \mathfrak{D} | \ X \subseteq D \} \\
		& = \bigcup \{ N(x) \ | \ x \in  X \}.
		\end{split}                           
		\end{equation}
		
	\end{definition}
	
	\noindent Next, we can see that definitions (2.11) of the lower and upper approximation operators is of the same form as (2.10) of the non-dual generalisation if we take $ R_s (x) = N(x)$. Observe that,
	\begin{equation*}
	\begin{split} 
	\lo (X) & = \{ x\ | \ R_s(x) \subseteq X \} \\
	& = \bigcup \{  R_s(x)  \ | \ R_s(x) \subseteq X \},
	\end{split}                    
	\end{equation*}	
	\begin{equation}
	\begin{split}
	\ut (X) & = \{ x \ | \ R_p(x) \cap X \neq \emptyset  \}\\
	& = \bigcup \{ R_s(x) \ | \ x \in X \}.
	\end{split}
	\end{equation}
	
	\noindent We can see that definition given by (11)  looks  similar to definition given by (9) by setting $N(x)$ as $R_s(x).$ However, in general $R_s(x)$ cannot be considered a neighbourhood of $x$ since we can show from Definition 2.4.1 that $N(x)$ seen as a relation on $V$ is reflexive and transitive, i.e. a pre-order. Hence, we consider the case of $R$ being a pre-order, i.e. $R_{rt},$ and we can set $ N(x) = R_{rt}(x).$  Then, we can use Table 2 to see which properties hold.  From the table we see that all of the usual rough set operator properties except duality holds for $R_{rt}.$ This accounts for the observation of Samanta and Chakraborty in \cite{PandC}  that the operators, $\underline{C_t}$ and $\overline{C_t}$ satisfy all of the rough set properties examined except duality. The reason is that, using the form of approximation operators given in Equation (9), any pre-order would satisfy at least those properties  satisfied by $R_{rt}$ in  Table 2.
	
	\subsubsection{Interpreted Logical Connection}
	
	Consider the case of the pre-order relation being an implication relation on a set of propositions $P.$ Then, the lower approximation of a subset of $P_1 \subseteq P$ say,  corresponds to the union of maximal theories contained in $P_1$, while the upper approximation corresponds to the smallest theory which contains $P_1$ i.e. its deductive closure.

	\section{Conclusion} We have documented and examined an interesting generalisation of rough sets as evidenced by the fact that it satisfies most of the usual rough set properties except for duality. A special case of this extension came up in the paper of  Samanta and Chakraborty, see \cite{PandC}, in which they observed that the respective operator in that paper satisfied almost  all of the considered properties and they remarked that it deserved further investigation.  In this direction, we gave at least the start of an explanation for their observation by placing it in this more general context. Then,  we showed an interpreted connection between this extension and certain logical notions. These type of links are especially helpful in forming a coherent map of the mass of existing literature out there. Furthermore, since this generalisation behaves quite nicely, it is interesting in its own right and worthy of further attention and study.

	\bibliographystyle{plain}
	
	\bibliography{mybib3b}

\end{document}